\documentclass[11pt,letterpaper]{article}
\usepackage{graphicx}
\usepackage{epsfig,subfigure}
\usepackage{amssymb}
\usepackage{amsthm}
\usepackage{amsfonts}
\usepackage{amsmath}
\usepackage[margin=1.2in]{geometry}
\RequirePackage{natbib}
\RequirePackage[colorlinks,citecolor=blue,urlcolor=blue]{hyperref}
\usepackage{algorithm}
\usepackage{algorithmic}
\usepackage{multicol}
\usepackage{multirow}
\usepackage{color,soul}

\usepackage{mkolar_definitions}

\numberwithin{equation}{section}
\theoremstyle{plain}

\let\hat\widehat
\let\tilde\widetilde

\newcommand{\bdb}{{B\Delta B^*}}

\title{\huge Recovering Block-structured Activations Using Compressive
  Measurements}
\date{}

\author{
Sivaraman Balakrishnan\thanks{Language Technology Institute, Carnegie
  Mellon University
; e-mail: {\tt
    sbalakri@cs.cmu.edu}.}, ~~Mladen Kolar\thanks{Machine Learning Department, Carnegie Mellon
  University
; e-mail: {\tt
  mladenk@cs.cmu.edu}.}, ~~Alessandro Rinaldo\thanks{Department of Statistics, Carnegie Mellon
  University
; e-mail: {\tt
  arinaldo@stat.cmu.edu}.}, ~ and ~ Aarti Singh\thanks{Machine Learning Department, Carnegie Mellon
  University
; e-mail: {\tt
  aarti@cs.cmu.edu}.}}

\bibpunct{(}{)}{,}{a}{,}{,}

\begin{document}

\maketitle

\begin{abstract}
  We consider the problems of detection and localization of a
  contiguous block of weak activation in a large matrix, from a small
  number of noisy, possibly adaptive, compressive (linear)
  measurements.  This is closely related to the problem of compressed
  sensing, where the task is to estimate a sparse vector using a small
  number of linear measurements.  Contrary to results in compressed
  sensing, where it has been shown that neither adaptivity nor
  contiguous structure help much, we show that for reliable
  \emph{localization} the magnitude of the weakest signals is strongly
  influenced by both structure and the ability to choose measurements
  adaptively while for \emph{detection} neither adaptivity nor
  structure reduce the requirement on the magnitude of the signal.  We
  characterize the precise tradeoffs between the various problem
  parameters, the signal strength and the number of measurements
  required to reliably detect and localize the block of
  activation. The sufficient conditions are complemented with
  information theoretic lower bounds.
\end{abstract}


\section{Introduction}

Compressive measurements provide a very efficient means of recovering
signals that are sparse in some basis or frame. Specifically, several
papers, including
\citet[][]{candies.tao.06,donoho:06,candes07dantzig}, and
\citet{candes.wakin} have shown that it is possible to recover, in an
$\ell_2$ sense, a $k$-sparse vector in $n$ dimensions using only
$\Ocal(k \log n)$ incoherent compressive measurements, instead of
measuring all of the $n$ coordinates. This is a novel and important
paradigm with applications in a wide range of scientific areas. Along
with $\ell_2$ recovery, researchers have also considered the problems
of \emph{detection} and \emph{localization} of a sparse signal
corrupted by additive noise, the former task logically preceding the
latter.  The problem of detection is to test whether all components of
the vector are zero. \citet[][]{duarte06,haupt_det07} and
\citet[]{ariascastro12} studied detection of sparse vectors from
compressive measurements, while \citet{arias2011global} identifies
conditions for successful detection in sparse linear regression (see
also \citet{ingster2010detection}).  The problem of localization is to
identify coordinates of the non-zero elements of a
signal. \citet{wainwright06sharp} and \citet[]{wainwrightlimits}
studied information theoretic limits and localization properties of
the LASSO procedure.  More recently, researchers have contributed two
important refinements: 1) by considering a sparse \emph{structured}
signal (such as a signal consisting of adjacent coordinates or a
block) \citep{modelcs,castrofundamental,haupttree} and 2) by allowing
for the possibility of taking adaptive measurements, i.e., where
subsequent measurements are designed based on past observations
\citep[see, e.g.,
][]{candeshowwell,castrofundamental,hauptcds,davenportcbs,nowakcbs}.
However, almost all of this work has been focused on recovery or
detection of (structured or unstructured) sparse data {\em vectors}
from (passive or adaptive) compressed measurements.

In this work we focus on the unexplored problems of detection and
localization for data matrices from compressive measurements. We are
concerned with signals that are both sparse and highly structured,
taking the form of a sub-matrix of a larger matrix with contiguous row
and column indices. Data matrices have been considered in the context
of low-rank matrix completion \citep[see,
e.g.,][]{negahbanmc,tsybakovmc}, where recovery in Frobenius norm is
studied. The problems of detection and localization for data matrices
that are observed directly were studied previously. See, for example,
\citet{sun10maximal, kolarBRS11, butucea11detection, butucea13sharp,
  Bhamidi12energy}.  However, compressive measurement schemes were not
investigated.  If the activation is unstructured, the treatment of
data matrices is exactly equivalent to the treatment of data vectors.
However, in the structured case the problem is rather different, as we
will show.  Data matrices with signals that are both sparse and highly
structured form a natural model for several real-world activations
such as when we have a group of genes (belonging to a common pathway
for instance) co-expressed under the influence of a set of similar
drugs \citep{genebiclus}, or when we have groups of patients
exhibiting similar symptoms \citep{asthma}, or when we have sets of
malware with similar signatures \citep{bitshred}, etc. However, in
many of these applications, it is difficult to measure, compute or
store all the entries of the data matrix. For example, measuring
expression levels of all genes under all possible drugs is expensive,
or recording the signatures of each individual malware is
computationally demanding as it might require stepping through the
entire malware code.  However, if we have access to linear
combinations of matrix entries (i.e. compressive measurements) such as
combined expression of multiple genes under the influence of multiple
drugs then we might need to only make and store few such measurements,
while still being able to infer the existence or location of the
activated block of the data matrix.  Thus, the goal is to detect or
recover the activated block (set of co-expressed genes and drugs or
malware with similar signatures) using only few compressive
measurements of the data matrix, instead of observing the entire data
matrix directly.  We consider both the passive (non-adaptive) and
active (adaptive) measurements. The non-adaptive measurements are
random or pre-specified linear combinations of matrix entries. In
other cases, such as mixing drugs, we might be able to adapt the
measurement process by using feedback to sequentially design linear
combinations that are more informative.

Extensions to a setup where there is a non-contiguous sub-matrix or
block of activation are also interesting, but beyond the scope of this
paper.  \citet[][]{sun10maximal, butucea11detection, butucea13sharp,
  Bhamidi12energy,kolarBRS11} study a problem where a large noisy
matrix is observed directly, i.e., not through compressed
measurements, and the block of activation is non-contiguous. In such a
setting, tight upper and lower bounds are derived for the localization
problem. However, passive and adaptive \emph{compressive} measurement
schemes were not investigated.

{\bf Summary of our contributions.}  Using information theoretic
tools, we establish \emph{lower bounds} on the minimum number of
compressive measurements and the weakest signal-to-noise ratio (SNR)
needed to detect the presence of an activated block of positive
activation, as well as to localize the activated block, using both
non-adaptive and adaptive measurements.  We also demonstrate minimax
optimal \emph{upper bounds} through detectors and estimators that can
guarantee consistent detection and localization of weak
block-structured activations using few non-adaptive and adaptive
compressive measurements.

Our results indicate that adaptivity and structure play a key role and
provide significant improvements over non-adaptive and unstructured
cases for localization of the activated block in the data matrix
setting. This is unlike the vector case where contiguous structure and
adaptivity have been shown to provide minor, if any, improvement.  We
describe the results for the sparse vector case in related work
section below.  A summary of the SNR needed for detection and
localization of an unstructured sparse vector using passive and
adaptive compressive measurements is given in
Table~\ref{tab:summary_vec}.

In our setting we take compressive measurements of a data
\emph{matrix} of size $n = (n_1 \times n_2)$, the activated block is
of size $k = (k_1 \times k_2)$, with minimum SNR per entry of
$\mu/\sigma$, and we have a budget of $m$ compressive measurements
with each measurement matrix constrained to have unit Frobenius norm.
Table~\ref{tab:summary} describes our main findings (for the case when
$n_1=n_2$ and $k_1=k_2$ and paraphrasing for clarity) and compare the
scalings under which passive and active, detection and localization
are possible.

\begin{table}[t]
\centering
\caption{Summary of known results for the sparse vector case, 
where the length of the vector is $n$ and the number of active elements 
is $k$. The number of measurements is $m$ and $\mu/\sigma$ represents
SNR per
element
of the activated elements.}
\begin{tabular}{c|c|c ll}
      &   Detection    & \multicolumn{2}{c}{Localization} \\ \hline
	\multirow{3}{*}{Passive}  &  
	\multirow{3}{*}{$\frac{\mu}{\sigma} \asymp  \sqrt{\frac{n}{mk^2}}$} &  
	\multirow{2}{*}{$ \frac{\mu}{\sigma} \asymp  \sqrt{\frac{n\log n}{m}},$} & 
	\multirow{2}{*}{\citet{wainwrightlimits}} \\ 
&   & \\
	& &	$m \succ k \log n $\\
\cline{1-1}\cline{3-4}
	\multirow{3}{*}{Active}   & 
	\multirow{3}{*}{\citet{ariascastro12}} &  
	\multirow{3}{*}{$\frac{\mu}{\sigma} \asymp \sqrt{\frac{n}{m}}$} & 
	\citet{castrofundamental}\\
	& & & \citet{davenportcbs}\\
	& & & \citet{nowakcbs} \\
\hline
\end{tabular}
\label{tab:summary_vec}
\end{table}

\begin{table}[t]

\centering
\caption{Summary of main findings for the case when
$n = n_1 \times n_2 \ (n_1=n_2)$ and $k = k_1 \times k_2 \ (k_1=k_2)$, 
where the size of the matrix is $n_1
  \times n_2$ and the size of the activation block is $k_1 \times k_2$. The
  number of measurements is $m$ and $\mu/\sigma$ represents SNR per
  element of the activated block.}
\begin{tabular}{c|c|cl}
      &   Detection    & Localization \\ \hline
\multirow{2}{*}{Passive}  &  \multirow{2}{*}{$\frac{\mu}{\sigma} \asymp  \sqrt{\frac{n_1n_2}{mk_1^2k_2^2}} $} &  \multirow{2}{*}{$ \frac{\mu}{\sigma} \asymp  \sqrt{\frac{n_1n_2}{m\min(k_1,k_2)}}  $}
    & \multirow{2}{*}{Theorems \ref{thm:passivelb} and \ref{thm:passiveub}} \\ 
&   & & \\
\cline{1-1}\cline{3-4}
\multirow{2}{*}{Active}   & \multirow{2}{*}{Theorems \ref{thm:detectionlb} and \ref{thm:detectionub}}
 &  \multirow{2}{*}{$\frac{\mu}{\sigma} \asymp \frac{1}{\sqrt{m}}\max \big( \sqrt{\frac{n_1n_2}{k_1^2k_2^2}}, \frac{1}{\sqrt{\min(k_1,k_2)}}\big) $ }
 &  \multirow{2}{*}{Theorems \ref{thm:activelb} and \ref{thm:activeub}}\\
 &  &  &  \\
\hline
\end{tabular}
\label{tab:summary}
\end{table}

For detection, akin to the vector setting, structure and adaptivity
play no role.  The structured data matrix setting requires an SNR
scaling of $\sqrt{n_1 n_2/(m k_1^2k_2^2)}$ for both non-adaptive and
adaptive cases, which is same as the SNR needed to detect a
$k_1k_2$-sparse non-negative vector of length $n_1n_2$ as demonstrated
in \citet{ariascastro12}. Thus, the structure of the activation
pattern as well as the power of adaptivity offer no advantage in the
detection problem.

For localization of the activated block, the structured data matrix
setting requires an SNR scaling as $\sqrt{n_1 n_2/(m \min(k_1,k_2))}$
using non-adaptive compressive measurements.  In contrast, the
unstructured setting requires a higher SNR of $\sqrt{n_1
  n_2\log(n_1n_2)/m}$ where $m \geq k_1k_2 \log(n_1n_2)$ as
demonstrated in \citet{wainwrightlimits}. Structure, without
adaptivity already yields a factor of $\sqrt{\min\rbr{k_1,k_2}}$
reduction in the smallest SNR that still allows for reliable
localization.  Moreover, adaptivity in the compressive measurement
design yields further improvements: with adaptive measurements,
identifying the activated block requires a much weaker SNR of
\[ 
\max(\sqrt{n_1n_2/(mk_1^2k_2^2)},\sqrt{1/(m\min(k_1,k_2))})
\]
for the weakest entry in the data matrix. In contrast, for the sparse
vector case, \citet{castrofundamental} showed that adaptive
compressive measurements cannot localize the non-zero components if
the SNR is smaller than $\sqrt{n_1n_2/m}$. A matching upper bound was
provided using compressive binary search in \citet{davenportcbs} and
\citet{nowakcbs} for localization of a single non-zero entry in the
vector. Thus, exploiting structure of the activations and designing
adaptive linear measurements can both yield significant gains if the
activation corresponds to a contiguous block in a data matrix.

{\bf Related Work.}  Our work builds on a number of fairly recent
contributions on detection, localization and recovery of a sparse and
weak unstructured signal by adaptive compressive measurements.  In
\citet{castrofundamental}, the authors show that the adaptive
compressive scheme offers improvements over the passive scheme which,
in terms of the mean-squared error (MSE) and localization, are limited
to a $\log(n)$ factor. The authors also provide a general proof
strategy for minimax analysis under adaptive
measurements. \citet{ariascastro12} further applies this strategy to
the problem of detection of an unstructured and structured sparse and
weak vector signal under compressive adaptive measurements.
\citet{nowakcbs} shows that a compressive version of standard binary
search achieves minimax performance for localization in a one-sparse
vector.  The work of \citet{wainwrightlimits} which is based on
analyzing the performance of an exhaustive search procedure under
passive measurements, is relevant to our analysis of passive
localization. Our analysis provides a generalization of these results
to the case of a {\it structured} signal embedded as a small
contiguous block in a large matrix.

While in this paper we focus on detection and localization, some other
papers have considered estimation of sparse vectors in the MSE sense
using adaptive compressive measurements. For example,
\citet{castrofundamental} establishes fundamental lower bounds on the
MSE in a linear regression framework, while \citet{hauptcds}
demonstrates upper bounds using compressive distilled sensing.
\citet{modelcs} and \citet{haupttree} have analyzed different forms of
structured sparsity in the vector setting, e.g. if the non-zero
locations in a data vector form non-overlapping or
partially-overlapping groups or are tree-structured.  Finally,
\citet{negahbanmc} and \citet{tsybakovmc} have considered a
measurement model identical to ours in the setting of low-rank matrix
completion, but in that setting the matrix under consideration is not
assumed to be a structured sparse matrix and the theoretical
guarantees are with respect to the Frobenius norm. Furthermore,
\cite{kolarBRS11} illustrate that penalization using the sum of
nuclear and $\ell_1$ norm cannot be used for localization in a related
model.

When data matrix is observed directly, \citet{butucea11detection}
study the problem of detection, while \citet{kolarBRS11} and
\citet{butucea13sharp} study the problem of
localization. \citet{sun10maximal} and \citet{Bhamidi12energy}
characterize largest average submatrices of the data matrix under the
null hypothesis that the signal is not present. Results in those
papers do not carry over to a setting where a data matrix is accessed
through compressive measurements, as already seen in the vector case
\citep{ariascastro12}.

The rest of this paper is organized as follows. We describe the
problem set up and notation in Section~\ref{sec:notat}. We study the
detection problem in Section \ref{sec:detection}, for both adaptive
and non-adaptive schemes. Section \ref{sec:passive_loc} is devoted to
the non-adaptive localization, while Section \ref{sec:active_loc} is
focused on adaptive localization. Finally, in Section
\ref{sec:experiments} we present and discuss some simulations that
support our findings. The proofs are given in the Appendix.

{\bf Notatation.} In this paper we denote $[n]$ to be the set
$\{1,\ldots,n\}$.  For a vector $\ab \in \RR^n$, we denote ${\rm
  supp}(\ab) = \{j\ :\ a_j \neq 0\}$ the support set, $\norm{\ab}_q$,
$q \in [1,\infty)$, the $\ell_q$-norm defined as $\norm{\ab}_q =
(\sum_{i\in[n]} |a_i|^q)^{1/q}$ with the usual extensions for $q \in
\{0,\infty\}$, that is, $\norm{\ab}_0 = |{\rm supp}(\ab)|$ and
$\norm{\ab}_\infty = \max_{i\in[n]}|a_i|$. For a matrix $\Ab \in
\RR^{n_1\times n_2}$, we denote $\norm{\Ab}_F$ the Frobenius norm
defined as $\norm{\Ab}_F = (\sum_{i\in[n_1],j\in[n_2]}
a_{ij}^2)^{1/2}$. For two sequences $\{a_n\}$ and $\{b_n\}$, we use
$a_n = \Ocal(b_n)$ to denote that $a_n < Cb_n$ for some finite
positive constant $C$.  We also denote $a_{n} = \Ocal(b_{n})$ to be
$b_{n}\gtrsim a_{n}$. If $a_{n} = \Ocal(b_{n})$ and $b_{n} =
\Ocal(a_{n})$, we denote it to be $a_{n}\asymp b_{n}$. The notation
$a_n = o(b_n)$ is used to denote that $a_nb_n^{-1}\rightarrow 0$.


\section{Preliminaries}
\label{sec:notat}

Let $A \in \RR^{n_1\times n_2}$ be a signal matrix with unknown
entries. We are interested in a highly {\it structured } setting where
a {\it contiguous} block of the matrix $A$ of size $(k_1\times k_2)$
has entries all equal to $\mu > 0$, while all the other elements of
$A$ are equal to zero. We denote the coordinate set of all contiguous
blocks, of size $k_1\times k_2$ with
\begin{equation}
\label{eq:bcal}
\Bcal = \left\{ I_r \times I_c\ :\ 
\begin{array}{l}
I_r \text{ and }I_c \text{ are contiguous subsets of } [n_1] 
\text{ and } [n_2], \\
|I_r|=k_1, |I_c|=k_2  
\end{array}
\right\}.
\end{equation}
Then $A = (a_{ij})$ with $a_{ij} = \mu \ind\{(i,j)\in B^*\}$ for some
(unknown) $B^* \in \Bcal$, where $\ind$ is the indicator
function. Some of our results extend to the case when the activation
is positive, but not constant on $B^*$, as we discuss below.
Note that we assume the size $(k_1 \times k_2)$ is known.

We consider the following observation model under which $m$ noisy
linear measurements of $A$ are available
\begin{equation}
\label{eqn:model}
y_i = \tr(AX_i) + \epsilon_i, \quad i = 1,\ldots,m,
\end{equation}
where $\epsilon_1,\ldots,\epsilon_m \iidsim \Ncal(0,\sigma^2)$, with
$\sigma > 0$ known, and the sensing matrices $(X_i)_{i\in[m]}$ are
normalized to satisfy either $\| X_i \|_F \leq 1$ or $\mathbb{E}\|
X_i\|_F^2 = 1$, i.e., every measurement has the same amount of
energy. These are similar assumptions as made in \cite{davenportcbs}
and \cite{candeshowwell}.

Under the observation model in Eq.~\eqref{eqn:model}, we study two
tasks: (1) detecting whether a contiguous block of positive signal
exists in $A$ and (2) identifying the block $B^*$, that is, the
localization of $B^*$. We develop efficient algorithms for these two
tasks that provably require the smallest number of measurements, as
explained below. The algorithms are designed for one of two
measurement schemes: (1) the measurement scheme can be implemented in
an adaptive or sequential fashion, that is, actively, by letting each
$X_i$ to be a (possibly randomized) function of $(y_j,
X_j)_{j\in[i-1]}$, and (2) the measurement matrices are chosen all at
once or ignoring the outcomes in previous measurements, that is,
passively.

{\bf Detection.} The detection problem concerns checking whether a
positive contiguous block exists in $A$. As we will show later, we can
detect the presence of a contiguous block with a much smaller number
of measurements than is required for localizing its position.
Formally, detection is a hypothesis testing problem with a composite
alternative of the form
\begin{equation}
\label{eq:detection_hypothesis}
\begin{array}{cl}
H_0 \colon & A = 0_{n_1 \times n_2} \\
H_1 \colon & A = (a_{ij}) \text{ with } a_{ij} = \mu \ind_{\{(i,j)\in
  B^*\}},\ B^* \in \Bcal.
\end{array}  
\end{equation}

A test $T$ is a measurable function of the observations
$(y_i)_{i\in[m]}$ and the measurements matrices $(X_i)_{i\in[m]}$,
which takes values in $\{ 0, 1\}$, with $T= 1$ if the null hypothesis
is rejected and $T=0$ otherwise.  For any test $T$, we define its risk
as
\[
R^{\rm det}(T)  \equiv \PP_0 \left[ T\big((y_i,X_i)_{i\in[m]}\big) = 1\right] + 
      \max_{B^* \in \mathcal{B}} \PP_{B^*} \left[T\big((y_i,X_i)_{i\in[m]}\big) = 0\right],
\] 
where $\PP_0$ and $\PP_B$ denote the joint probability distributions
of $\big((y_i,X_i)_{i\in[m]}\big)$ under the null hypothesis and when
the activation pattern is $B$, respectively. The risk $R(T)$ measures
the maximal sum of type I and type II errors over the set of
alternatives. The overall difficulty of the detection problem is
quantified by the {\it minimax risk}
\[
R^{\rm det} \equiv \inf_{T} R^{\rm det}(T),
\]
where the infimum is taken over all tests. For a sufficiently small
SNR, the minimax risk is bounded away from zero by a large constant,
which implies that no test can distinguish $H_0$ from $H_1$. In
Section \ref{sec:detection} we will precisely characterize the
boundary for SNR $\smallfrac{\mu}{\sigma}$ below which no test can
distinguish $H_0$ and $H_1$.

{\bf Localization.}  The localization problem concerns the recovery of
the true activation pattern~$B^*$. Let $\Psi$ be an estimator of
$B^*$, i.e., a measurable function of $(y_i, X_i)_{i\in[m]}$ taking
values in $\Bcal$. We define the risk of any such estimator as 
\[
R^{\rm loc}(\Psi) = \max_{B^* \in \mathcal{B}} P_{B^*} \left[ \Psi\big((y_i,X_i)_{i\in[m]}\big) \neq B^* \right],
\]
while the {\it minimax risk} 
\[
R^{\rm loc} \equiv \inf_{\Psi} R^{\rm loc}(\Psi)
\]
of the localization problem is the minimal risk over all such
estimators $\Psi$.  Like in the detection task, the minimax risk
specifies the minimal risk of any localization procedure. By standard
arguments, the evaluation of the minimax localization risk also
proceeds by first reducing the localization problem to a hypothesis
testing problem \citep[see, e.g.,][for
details]{tsybakov09introduction}.

Below we will provide a sharp characterization, through information
theoretic lower bounds and tractable estimators, of the minimax
detection and localizations risks as functions of tuples of
$(n_1,n_2,k_1,k_2,m, \mu,\sigma)$ and for both the active and passive
sampling schemes.  Our results identify precisely both the minimal SNR
given a budget of $m$ possibly adaptive measurements, and the minimal
number of measurements $m$ for a given SNR in order to achieve
successful detection and localization.

Along with a careful and detailed minimax analysis, we also describe
procedures for detection and localization in both the active
and passive case whose risks match the minimax rates.


\section{Detection of contiguous blocks}
\label{sec:detection}
In this section, we derive minimax rates for detection. 

\subsection{Lower bound}

The following theorem gives a lower bound on the SNR needed to
distinguish $H_0$ and $H_1$. 

\begin{theorem}
\label{thm:detectionlb}
Fix any $0 < \alpha < 1$. Based on $m$ (possibly adaptive)
measurements, if  
$$ \mu \leq \sigma (1 - \alpha) \sqrt{\frac{16(n_1 - k_1) (n_2 -
    k_2)}{mk_1^2k_2^2}},
$$
then $R^{\rm det} \geq \alpha$.
\end{theorem}

The lower bound on possibly \emph{adaptive} procedures is established
by analyzing the risk of the (optimal) likelihood ratio test under a
uniform prior over the alternatives. Careful modifications of standard
arguments are necessary to account for adaptivity.  We closely follow
the approach of Arias-Castro \cite{ariascastro12} who established the
analogue of Theorem~\ref{thm:detectionlb} in the vector setting.
 
 \subsection{Upper bound}
 
We now demonstrate the sharpness of the result established in the
previous section. We choose the sensing matrices passively as $X_i =
(n_1n_2)^{-1/2} \one_{n_1}\one_{n_2}'$ and consider the following test
\begin{equation}
\label{eq:test}
T\big((y_i)_{i\in[m]}\big) = \ind\Big\{\sum_iy_i >
\sigma\sqrt{2m\log(\alpha^{-1})}\Big\}.
\end{equation}
\begin{theorem}
\label{thm:detectionub}
Assume that $k_1 \leq c n_1$ and $k_2 \leq c n_2$ for some $c\in(0,1)$.
If 
\[ \mu \geq \sigma \sqrt{\smallfrac{8n_1n_2\log( \alpha^{-1}) }{m
    k_1^2k_2^2}} 
\]
 then $R^{\mathrm{det}}(T) \leq \alpha$, where $T$ is the test
defined in Eq.~\eqref{eq:test}.
\end{theorem}
The results of Theorem~\ref{thm:detectionlb} and
Theorem~\ref{thm:detectionub} establish that the minimax rate for
detection under the model in Eq.~\eqref{eqn:model} is $ \mu \asymp
\sigma(k_1k_2)^{-1}\sqrt{m^{-1}n_1n_2}$, under the (mild) assumption
that $k_1 \leq c n_1$ and $k_2 \leq c n_2$ for any constant $0 < c <
1$. It is worth pointing out that the structure of the activation
pattern \emph{does not} play any role in the minimax detection
problem, since the rate matches the known bounds for detection in the
unstructured vector case \cite{ariascastro12}. We will contrast this
to the localization problem below. Furthermore, the procedure that
achieves the adaptive lower bound (upto constants) is non-adaptive,
indicating that adaptivity can not help much in the detection problem.

We also note that results established in this section continue to hold
when the activation is positive, but not constant on $B^*$, with
$\min_{(i,j) \in B^*} a_{ij}$ replacing $\mu$.


\section{Localization from passive measurements}
\label{sec:passive_loc}

In this section, we address the problem of estimating a contiguous
block of activation $B^*$ from noisy linear measurements as in
equation~\eqref{eqn:model}, when the measurement matrices
$(X_i)_{i\in[m]}$ are independent with i.i.d.~entries having a $
\Ncal(0, (n_1n_2)^{-1})$ distribution.  The variance of the elements
is set so that $\EE\norm{X_i}_F^2 = 1$.

\subsection{Lower bound}
The following theorem gives a lower bound on the SNR needed for any
procedure to localize $B^*$. 
\begin{theorem}
\label{thm:passivelb}
There exist positive constants $C, \alpha > 0$ independent of the
problem parameters $(k_1,k_2,n_1,n_2)$, such that if
\begin{equation*}
  \label{eq:passive_loc:mu_lower_bound}
  \mu \leq C\sigma
  \sqrt{\frac{n_1n_2}{m}
    \max\rbr{\frac{1}{\min(k_1,k_2)},
    \frac{\log\max(n_1-k_1, n_2-k_2)}{k_1k_2}
    }},
\end{equation*}
then $ R^{\mathrm{loc}} \geq \alpha > 0$.
\end{theorem}
The proof is based on a standard technique described in Chapter 2.6 of
\citet{tsybakov09introduction}.  We start by identifying a subset of
matrices that are hard to distinguish.  Once a suitable finite set is
identified, tools for establishing lower bounds on the error in
multiple-hypothesis testing can be directly applied.  These tools only
require computing the Kullback-Leibler (KL) divergence between the
induced distributions, which in our case are two multivariate normal
distributions.

The two terms in the lower bound feature two aspects of our
construction, the first term arises from considering two matrices that
overlap considerably, while the second term arises from considering
matrices that do not overlap at all of which there are possibly a very
large number.  These constructions and calculations are described in
detail in the Appendix.

\subsection{Upper bound}
We will investigate a procedure that searches over all contiguous
blocks of size $(k_1 \times k_2)$ as defined in Eq.~\eqref{eq:bcal}
and outputs the one minimizing the squared error. Specifically, let
the loss function $f : \Bcal \mapsto \RR$ be
\begin{equation}
  \label{eq:score_function}
  f(B) := \min_\mu\ \sum_{i \in [m]} 
  \Big(\mu \sum_{(a,b) \in B} X_{i,ab} - y_i\Big)^2,
\end{equation}
where $X_{i,ab}$ denotes element in row $a$ and column $b$ of the $i^{\mathrm{th}}$
sensing matrix.
Then the estimated block $\hat B$ is defined as
\begin{equation}
  \label{eq:estim_block}
  \hat{B} := \argmin_{B \in \Bcal} f(B).
\end{equation}
Note that the minimization problem above requires solving $O(n_1 n_2)$
univariate regression problems and can be implemented efficiently for
reasonably large matrices.

The following result characterizes the SNR needed for $\hat B$ to
correctly identify $B^*$.
\begin{theorem}
  \label{thm:passiveub}
  There exist positive constants $C_1, C_2 > 0$ independent of the
  problem parameters $(k_1,k_2,n_1,n_2)$, such that if $m \geq C_1
  \log \max(n_1 - k_1, n_2 - k_2)$ and
\begin{equation*}
  \label{eq:passive_loc:mu_upper_bound}
  \mu \geq C_2 \sigma \sqrt{\frac{n_1n_2}{m} \log(2/\alpha)
  \max\rbr{\frac{\log\max(k_1,k_2)}{\min(k_1,k_2)}, 
       \frac{\log\max(n_1-k_1,n_2-k_2)}{k_1k_2}}},
\end{equation*}
for $0 < \alpha \leq 1$, 
then $R^\mathrm{loc} (\hat{B}) \leq \alpha$, where $\hat B$ is defined in
Eq.~\eqref{eq:estim_block}.
\end{theorem}
Comparing to the lower bound in Theorem~\ref{thm:passivelb}, we
observe that the procedure outlined in this section achieves the lower
bound up to constants and a $\log\rbr{\max\rbr{k_1,k_2}}$
factor. Under the scaling $\max(k_1, k_2) \geq
\log \max(n_1 - k_1, n_2 - k_2)$, we obtain that the \emph{passive}
minimax rate for localization of the active blocks $B^*$ is $ \mu
\asymp \tilde{O} \big(\sigma\sqrt{(m\min(k_1,k_2))^{-1}n_1n_2} \big)$.
In this and subsequent uses, 
the $\tilde{O}$ notation hides a $\sqrt{\log \max(k_1,k_2)}$
factor.

This establishes that the SNR needed for passive localization is
considerably larger than the bound we saw earlier for passive
detection. This should be contrasted to the unstructured normal means
problem, where the bounds for localization and detection differ only
in constants \citep{donoho04higher}.

The block structure of the activation allows us, even in the passive
setting, to localize much weaker signals. A straightforward adaptation
of results on the LASSO \citep{wainwright06sharp} suggest that if the
non-zero entries are spread out (say at random) then we would require
$\mu \asymp \tilde{O} \left( \sigma \sqrt{\frac{n_1n_2}{m}} \right)$
for localization.

  One could extend the analysis in this section to data matrices
  with non-constant activation as in \cite{wainwrightlimits}.
  Furthermore, one can adapt to the unknown size of the activation
  block. In particular, one can perform exhaustive search procedure
  for all possible sizes of activation blocks. Let 
  $\Bcal_{k_1,k_2}$ denote the coordinate set of all contiguous blocks
  of size $k_1 \times k_2$. Then the estimated block 
  \[
    \hat B = \argmin_{B \in \cup_{k_1,k_2} \Bcal_{k_1, k_2}}\ f(B)
  \]
  adapts to the unknown size of the activation if the signal strength
  satisfies the condition in Theorem~\ref{thm:passiveub}. This can be
  verified by small modifications to the proof of
  Theorem~\ref{thm:passiveub}.

\subsubsection{The non-contiguous case}
Suppose that the block of activation $B^*$ belongs to the collection
$\tilde \Bcal$, where
\[
\tilde \Bcal = \{ I_r \times I_c:
I_r \subset [n_1], I_c \subset [n_2],
|I_r|=k_1, |I_c|=k_2  \},
\]
so that the activation block is not necessarily a contiguous
block. This collection contains less structure than the collection
$\Bcal$, but we can still localize much weaker signals compared to
completely unstructured case. Slight modification of
proofs\footnote{A sketch of the derivation is given in
  Appendix~\ref{sec:app:proof_sketch}} of Theorem~\ref{thm:passivelb}
and Theorem~\ref{thm:passiveub} yields the following. 
\begin{theorem}
  Let $\tilde{B} := \argmin_{B \in \tilde\Bcal} f(B)$. There exists
  a constant $C_1$ such that if the signal strength satisfies
  \begin{equation}
  \label{eq:upper_bound_bicluster_passive}
  \mu \geq C_1\sigma
  \sqrt{
    \frac{n_1n_2}{m}\log(2/\alpha)
      \frac{\log(n_1-k_1)(n_2-k_2)}{k_1+k_2}
  },
  \end{equation}
  then  $R^{\mathrm{loc}}(\tilde B) \leq \alpha$, for any $0 < \alpha \leq 1$.

  Conversely, there exists constants $C_2, \alpha > 0$ such that if
\begin{equation}
  \label{eq:lower_bound_bicluster_passive}
  \mu \leq C_2 \sigma \sqrt{\frac{n_1n_2}{m}\max\rbr{
      \frac{\log(n_1-k_1)}{k_2},
      \frac{\log(n_2-k_2)}{k_1},
      \frac{\log {n_1 - k_1 \choose k_1}{n_2-k_2 \choose k_2} }{k_1k_2}
    }},
\end{equation}
then $ R^{\mathrm{loc}} \geq \alpha > 0$.
\end{theorem}

Therefore, we conclude that even without contiguous blocks, the
additional structure helps for the problem of localization.


\section{Localization from active measurements}
\label{sec:active_loc}

In this section, we study localization of $B^*$ using adaptive
procedures, that is, the measurement matrix $X_i$ may be a function of
$(y_j, X_j)_{j\in[i-1]}$.

\subsection{Lower bound}

A lower bound on the SNR needed for any active procedure to localize
$B^*$ is given as follows.
\begin{theorem}
\label{thm:activelb}
Fix any $0 < \alpha < 1$. Given $m$ adaptively chosen measurements,
if
\[
\mu  < \sigma (1 - \alpha) 
\max  \left( 
  \sqrt{\frac{2 \max(
      (n_1 - k_1) (n_2/2 - k_2),
      (n_1/2 - k_1) (n_2 - k_2) )}{mk_1^2 k_2^2}}, 
  \sqrt{
    \frac{8}{m \min(k_1,k_2)}}\right)
\]
then $ R^{\mathrm{loc}}\geq \alpha$.
\end{theorem}
The proof is based on information theoretic arguments applied to
specific pairs of hypotheses that are hard to distinguish. The two
terms in the lower bound reflect the two important
sources of hardness of the
problem of localization. The first term
reflects the difficulty of approximately localizing the block of
activation. This term grows at the same rate as the detection lower
bound, and its proof is similar. Given a coarse localization of the
block we still need to exactly localize the block. The hardness of
this problem gives rise to the second term in the lower bound. The
term is independent of $n_1$ and $n_2$ but has a considerably worse
dependence on $k_1$ and $k_2$.

\subsection{Upper bound}

\begin{table}[p]
\centering
\begin{minipage}{\textwidth}
\begin{algorithm}[H]
\label{alg:apploc}
\caption{Approximate localization}
{\fontsize{10}{10}\selectfont
\begin{algorithmic}
  \INPUT Measurement budget $m \geq \log p$, 
  ordered collection of size\footnote{We assume $p$ is dyadic to simplify
  our presentation of the algorithm.} $p$  of blocks $\mathcal{D}$ of size $(u_1 \times u_2)$ \\
  Initial support: $J_0^{(1)} \equiv \{1,\dots,p\}$, $s_0 \equiv \log p$ \\
  For each $s$ in $1,\ldots,\log_2p$
\begin{enumerate}
\item Allocate: $m_s \equiv \lfloor (m - s_0) s 2^{-s - 1}\rfloor + 1$
\item Split: $J_1^{(s)}$ and $J_2^{(s)}$, left and right half collections of blocks of $J_0^{(s)}$
\item Sensing matrix: $X_s = \sqrt{ \smallfrac{2^{-(s_0-s+1)}}{u_1u_2}} $ on $J_1^{(s)}$, 
$X_s = -\sqrt{ \smallfrac{2^{-(s_0-s+1)}}{u_1u_2}}$ on  $J_2^{(s)}$
and $0$ otherwise.
\item Measure: $y_i^{(s)} = \mathrm{tr} (AX_s) + z_i^{(s)}$ for $i \in [1,\ldots,m_s]$
\item Update support: $J_0^{(s+1)} = J_1^{(s)}$ if $\sum_{i=1}^{m_s}y_i^{(s)}>0$
and $J_0^{(s+1)} = J_2^{(s)}$ otherwise 
\end{enumerate}
\OUTPUT The single block in $J_0^{(s_0+1)}$.
\end{algorithmic}}
\end{algorithm}
\renewcommand\footnoterule{}
\end{minipage}
\end{table}

\begin{table}[p]
\begin{minipage}{\textwidth}
\begin{algorithm}[H]
\label{exactloc}
\caption{Exact localization (of columns)}
{\fontsize{10}{10}\selectfont
\begin{algorithmic}
\INPUT Measurement budget $m$, a sub-matrix $B \in \RR^{4k_1\times
  4k_2}$, success probability $\delta$ \\
  \begin{enumerate}
  \item Measure: $y_i^c = (4k_1)^{-1/2}\sum_{l=1}^{4k_1}B_{lc} + z_i^c$ for
$i=\{1,\ldots,m/5\}$ and $c\in\{1, k_2+1, 2k_2+1, 3k_2+1\}$ \\
\item Let $l = \argmax_c \sum_{i=1}^{m/5} y_i^c$, $r = l + k_2 $, $m_b = \lfloor \smallfrac{m}{6 \log_2 k_2} \rfloor$
 \item While $r - l \geq 1$
\begin{enumerate}
\item Let $c = \lfloor \smallfrac{r + l}{2}\rfloor$
\item Measure $y_i^c = (4k_1)^{-1/2}\sum_{l=1}^{4k_1}B_{lc} + z_i^c$
for $i=\{1,\ldots,m_b\}$ 
\item If\footnote{The exact constants appear in the proof of Theorem \ref{thm:activeub}.} $\sum_{i=1}^{m_b} y_i^c \geq \Ocal \left( \sqrt{\log \left( \frac{ \log k_2}{\delta} \right) \frac{m_b\sigma^2}{ \log k_2}}\right)$ then 
$l = c$, otherwise $r = c$.
\end{enumerate}
\end{enumerate}
\OUTPUT Set of columns $\{l-k_2+1, \ldots, l\}$. 
\end{algorithmic}}
\label{alg:exact}
\end{algorithm}
\renewcommand\footnoterule{}
\end{minipage}
\end{table}

The upper bound is established by analyzing the procedures described
in Algorithms 1 and 2 for approximate and exact localization.
Algorithm 1 is used to approximately locate the activation block, that
is, it locates a $8k_1 \times 8k_2$ block that contains the activation
block with high probability. The algorithm essentially performs
compressive binary search (\cite{davenportcbs}) on a collection of
non-overlapping blocks that partition the matrix. It is run on four
collections, $\mathcal{D}_1, \mathcal{D}_2, \mathcal{D}_3$ and
$\mathcal{D}_4$ defined as\footnote{For simplicity, we assume $n_1$ is
  a multiple of $2k_1$ and $n_2$ of $2k_2$}
\begin{eqnarray*}
\mathcal{D}_1 & \equiv & \left\{ B_{1,1} := [1,\ldots,2k_1] \times [1,\ldots,2k_2]
  , B_{1,2} := [2k_1+1,\ldots,4k_1] \right. 
 \times [1,\ldots,2k_2]  \\
 & &
\left. \ldots , B_{1,n_1n_2/4k_1k_2}
:= [n_1 - 2k_1, \ldots,n_1] \times [n_2-2k_2,\ldots,n_2] \right\} \\
\mathcal{D}_2 & \equiv & \left\{ B_{2,1} :=
  [k_1,\ldots,3k_1] \times [k_2,\ldots,3k_2] , B_{2,2} :=
  [3k_1+1,\ldots,5k_1] \times [k_2,\ldots,3k_2] \right. \\
 & & \left.  \ldots , B_{2,n_1n_2/4k_1k_2} :=
  [n_1-k_1,...,n_1,1,\ldots,k_1] \times [n_2
  -k_2,...,n_2,1,\ldots,k_2]\right\} \\
\mathcal{D}_3 & \equiv & \left\{ B_{3,1} :=
  [k_1,\ldots,3k_1] \times [1,\ldots,2k_2] , B_{3,2} :=
  [3k_1+1,\ldots,5k_1] \times [1,\ldots,2k_2] \right. \\
 & & \left.  \ldots , B_{3,n_1n_2/4k_1k_2} :=
  [n_1-k_1,...,n_1,1,\ldots,k_1] \times [n_2-2k_2,\ldots,n_2]\right\} 
  \end{eqnarray*}
  and
  \begin{eqnarray*}
  \mathcal{D}_4 & \equiv & \left\{ B_{4,1} :=
  [1,\ldots,2k_1] \times [k_2,\ldots,3k_2] , B_{4,2} :=
  [2k_1+1,\ldots,4k_1] \times [k_2,\ldots,3k_2] \right. \\
 & & \left.  \ldots , B_{4,n_1n_2/4k_1k_2} :=
  [n_1 - 2k_1, \ldots,n_1] \times [n_2
  -k_2,...,n_2,1,\ldots,k_2]\right\}.
\end{eqnarray*}
$\mathcal{D}_1$ is a partition of the matrix into disjoint blocks of
size $(2k_1 \times 2k_2)$, $\mathcal{D}_3$ is a similar partition
shifted down by $k_1$ rows, $\mathcal{D}_4$ is shifted to the right by
$k_2$ columns and $\mathcal{D}_2$ is both shifted down by $k_1$ rows
and to the right by $k_2$ columns.  Figure \ref{fig:blocks}
illustrates this.

Notice, that one of these collections must include a block that
contains the \emph{full} block of activation.  Algorithm 1 applied
four times returns four blocks, one of which as we show contains the
full activation block with high probability.

Algorithm 2 is used next to precisely locate the activation block
within one of the four coarser blocks identified by Algorithm 1.
Algorithm 2 itself works in several stages: in the first stage the
procedure measures a small number of columns, exactly one of which is
active, repeatedly, to identify the active column with high
probability.  The next stage finds the first non-active column to the
left and right by testing columns using a binary search (halving)
procedure. In this way, all the active columns are located.  Finally,
Algorithm 2 is repeated on the rows to identify the active rows.

The following theorem states that Algorithm 1 and Algorithm 2 succeed
in localization of the active block with high probability if the SNR is
large enough.
\begin{theorem}
\label{thm:activeub}
If
$$\mu \geq \sigma \sqrt{\log (1/\alpha)} ~~ \tilde{O} 
\left( \max \left( \sqrt{\frac{n_1n_2}{m k_1^2 k_2^2}},
    \sqrt{\frac{1}{\min(k_1,k_2)m}}\right)\right) $$ and $m \geq 3
\log (n_1 n_2)$ then $R(\hat{B}) \leq \alpha$, where
$\hat{B}$ is the block output by the algorithms.
\end{theorem}
As before, the $\tilde{O}$ hides a $\sqrt{\log \max(k_1,k_2)}$ factor,
and our upper bound matches the lower bound up to this factor.  It is
worth noting that for small activation blocks (when the first term
dominates) our active localization procedure achieves the
\emph{detection} limits. This is the best result we could hope
for. For larger activation blocks, the lower bound indicates that
\emph{no} procedure can achieve the detection rate. The active
procedure still remains significantly more efficient than the passive
one, and even in this case is able to localize signals that are weaker
by a (large) $\sqrt{n_1n_2}$ factor. This is not the case for
compressed sensing of vectors as shown in \citet{castrofundamental}.
The great potential for gains from adaptive measurements is clearly
seen in our model which captures the fundamental interplay between
\emph{structure} and \emph{adaptivity}.

\begin{figure}[t]
  \centering
  \includegraphics[width=0.7\columnwidth]{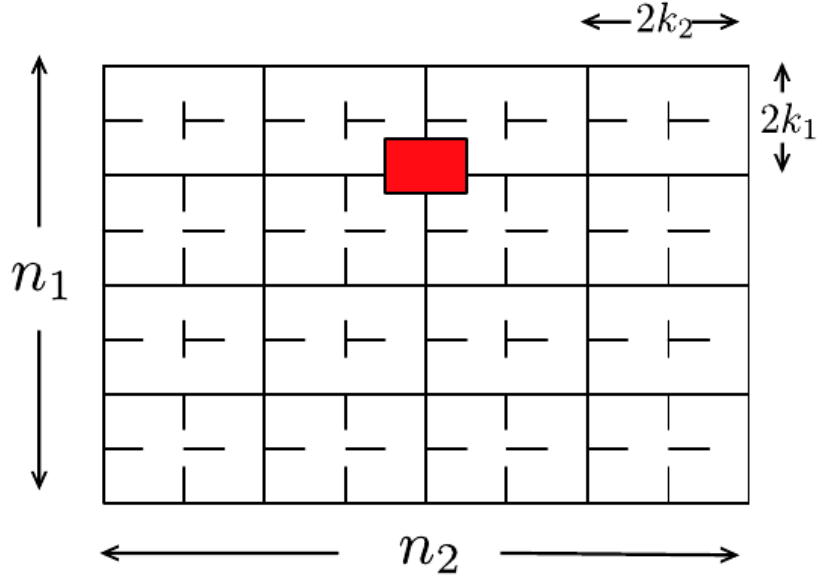}
  \caption{The collection of blocks $\mathcal{D}_1$ is shown 
  in solid lines and the collection $\mathcal{D}_2$ is shown
  in dashed lines. The collections $\mathcal{D}_3$ and $\mathcal{D}_4$
  overlap with these and are not shown.
  The $(k_1 \times k_2)$ block of activation is shown in red.}
  \label{fig:blocks}
\end{figure}

\section{Experiments}
\label{sec:experiments}
In this section, we perform a set of simulation studies to illustrate
finite sample performance of the proposed procedures.  We let $n_1 = n_2
= n$ and $k_1 = k_2 = k$.  Theorem~\ref{thm:passiveub} and
Theorem~\ref{thm:activeub} characterize the SNR needed for the passive
and active identification of a contiguous block, respectively.  We
demonstrate that the scalings predicted by these theorems are sharp by
plotting the probability of successful recovery against appropriately
rescaled SNR and showing that the curves for different values of $n$
and $k$ line up.

{\bf Experiment 1.}  Figure~\ref{fig:passive} shows the probability of
successful localization of $B^*$ using $\hat B$ defined in
Eq.~\eqref{eq:estim_block} plotted against $n^{-1}\sqrt{km}*{\rm
  SNR}$, where the number of measurements $m = 100$. Each plot in
Figure~\ref{fig:passive} represents different relationship between $k$
and $n$; in the first plot, $k = \Theta(\log n)$, in the second $k =
\Theta(\sqrt{n})$, while in the third plot $k = \Theta(n)$. The dashed
vertical line denotes the threshold position for the scaled SNR at
which the probability of success is larger than $0.95$. We observe
that irrespective of the problem size and the relationship between $n$
and $k$, Theorem~\ref{thm:passiveub} tightly characterizes the minimum
SNR needed for successful identification.

\begin{figure}[!h]
  \centering
  \includegraphics[width=\columnwidth]{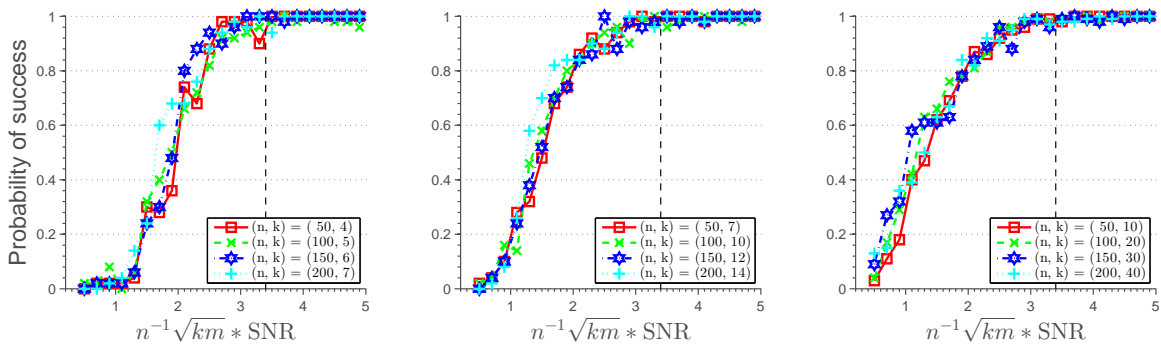}
  \caption{\small Probability of success with passive measurements
    (averaged over 100 simulation runs).}
  \label{fig:passive}
\end{figure}

{\bf Experiment 2.}  Figure~\ref{fig:active} shows the probability of
successful localization of $B^*$ using the procedure outlined in
Section~5.2., with $m = 500$ adaptively chosen measurements, plotted
against the scaled SNR.  The SNR is scaled by $n^{-1}\sqrt{m}k^2$ in
the first two plots where $k = \Theta(\log n)$ and $k =
\Theta(\sqrt{n})$ respectively, while in the third plot the SNR is
scaled by $\sqrt{mk/\log k}$ as $k = \Theta(n)$.  The dashed vertical
line denotes the threshold position for the scaled SNR at which the
probability of success is larger than $0.95$. We observe that
Theorem~\ref{thm:activeub} sharply characterizes the minimum SNR
needed for successful identification.

\begin{figure}[!h]
  \centering
  \includegraphics[width=\columnwidth]{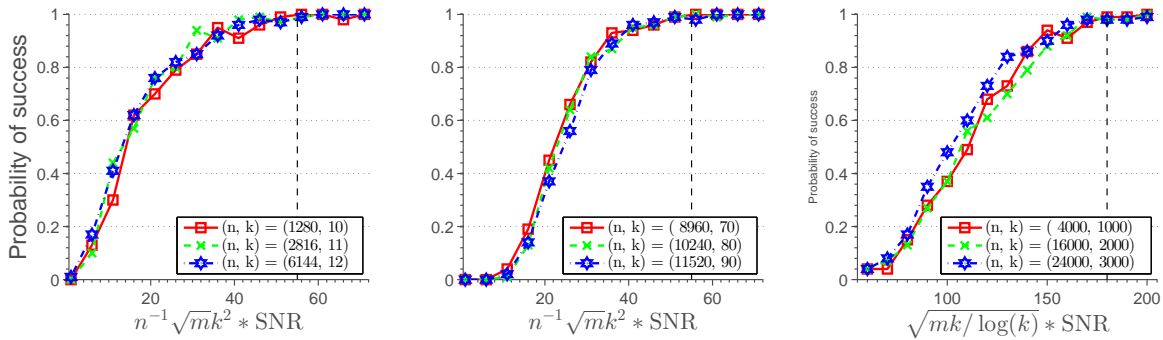}
  \caption{\small Probability of success with adaptively chosen measurements
    (averaged over 100 simulation runs).}
  \label{fig:active}
\end{figure}


\section*{Acknowledgements}

We would like to thank Larry Wasserman for his ideas, indispensable
advice and wise guidance. This research is supported in part by AFOSR
under grant FA9550-10-1-0382, NSF under grant IIS-1116458 and NSF
CAREER grant DMS 1149677.

\bibliography{biblio}

\clearpage
\appendix

\section{Proofs of Main Results}

In this appendix, we collect proofs of the results stated in the
paper. Throughout the proofs, we will denote $c_1, c_2, \ldots$
positive constants that may change their value from line to line.

\subsection{Proof of Theorem \ref{thm:detectionlb}}

We lower bound the Bayes risk of any test $T$. Recall, the null and
alternate hypothesis, defined in Eq.~\eqref{eq:detection_hypothesis},
\begin{equation*}
\begin{array}{cl}
H_0 \colon & A = 0_{n_1 \times n_2} \\
H_1 \colon & A = (a_{ij}) \text{ with } a_{ij} = \mu \ind_{\{(i,j)\in
  B\}},\ B \in \Bcal.
\end{array}  
\end{equation*}

We will consider a uniform prior over the alternatives $\pi$, and 
bound the average risk
\[
R_\pi(T) = \mathbb{P}_0[T = 1] 
  + \mathbb{E}_{A \sim \pi} \mathbb{P}_A[T = 0],
\]
which provides a lower bound on the worst case risk of $T$.

Under the prior $\pi$, the hypothesis testing becomes to distinguish
\begin{equation*}
\begin{array}{cl}
H_0 \colon & A = 0_{n_1 \times n_2} \\
H_1 \colon & A = (a_{ij}) \text{ with } a_{ij} = \EE_{B \sim \pi} \mu \ind_{\{(i,j)\in
  B\}}.
\end{array}  
\end{equation*}

Both $H_0$ and $H_1$ are simple and the likelihood ratio test is
optimal by the Neyman-Pearson lemma. The likelihood ratio is
\[
L \equiv 
\frac{\mathbb{E}_\pi \mathbb{P}_A [(y_i, X_i)_{i \in [m]}]}
     {\mathbb{P}_0 [(y_i, X_i)_{i \in [m]}]}
= \frac{ \mathbb{E}_\pi \prod_{i = 1}^m  \mathbb{P}_A [y_i | X_i] }
   {  \prod_{i = 1}^m  \mathbb{P}_0 [y_i | X_i] },
\]
where the second equality follows by decomposing the probabilities by
the chain rule and observing that $P_0 [X_i | (y_j, X_j)_{j \in
  [i-1]}] = P_A [X_i | (y_j, X_j)_{j \in [i-1]}]$, since the sampling
strategy (whether active or passive) is the same irrespective of the
true hypothesis.

The likelihood ratio can be further simplified as
\[
 L = \mathbb{E}_\pi \exp \left( \sum_{i=1}^m \frac{ 2 y_i
     \mathrm{tr}(AX_i) - \mathrm{tr}(AX_i)^2}{2 \sigma^2} \right).
\]

The average risk of the likelihood ratio test
\[
R_\pi(T) = 1 - \frac{1}{2} 
\norm{\mathbb{E}_\pi \mathbb{P}_A - \mathbb{P}_0}_{TV}
\]
is determined by the total variation distance between the mixture of
alternatives from the null.

By Pinkser's inequality \citet{tsybakov09introduction},
\[
\norm{\mathbb{E}_\pi \mathbb{P}_A - \mathbb{P}_0}_{TV} 
\leq \sqrt{KL(\mathbb{P}_0,\mathbb{E}_\pi \mathbb{P}_A)/2}
\]
and
\begin{align*}
KL(\mathbb{P}_0, \mathbb{E}_\pi \mathbb{P}_A) 
& = - \mathbb{E}_0 \log L \\
& \leq - \mathbb{E}_\pi \sum_{i = 1}^m  
   \mathbb{E}_0  \frac{ 2 y_i \mathrm{tr}(AX_i) - \mathrm{tr}(AX_i)^2}{2 \sigma^2} \\ 
& = \mathbb{E}_\pi \sum_{i=1}^m \mathbb{E}_0 \frac{\mathrm{tr}(AX_i)^2}{2 \sigma^2} \\
& \leq 
\frac{m}{2\sigma^2}
\sup_{\norm{X}_F\leq 1}\mathbb{E}_\pi \mathrm{tr}(AX_i)
:= \frac{m}{2\sigma^2} \norm{C}_{op},
\end{align*}
where the first inequality follows by applying the Jensen's inequality
followed by Fubini's theorem, and the second inequality follows using
the fact that $\norm{X_i}_F^2 = 1$, where $C \in \mathbb{R}^{n_1n_2
  \times n_1n_2}$.

To describe the entries of $C$, consider the invertible map $\tau$
from a linear index in $\{1, \ldots, n_1n_2\}$ to an entry of
$A$. Now, $C_{ii} = \mu^2 \mathbb{E}_\pi P_A[ A_{\tau(i)} = 1] $ and
$C_{ij} = \mu^2 \mathbb{E}_\pi P_A [A_{\tau(i)} = 1, A_{\tau(j)} = 1]$.

To bound the operator norm of $C$ we make two observations. Firstly,
because of the contiguous structure of the activation pattern, in any
row of $C$ there are at most $k_1k_2$ non-zero entries. Secondly, each
non-zero entry in $C$ is of magnitude at most $\mu^2 k_1k_2/(n_1 -
k_1)(n_2-k_2)$.

Now, note that
$$\norm{C}_{op} \leq \max_j \sum_k |C_{jk}| \leq \mu^2 k_1^2k_2^2/(n_1 - k_1)(n_2 - k_2)$$
from which we obtain a bound on the $KL$ divergence.

Now, this gives us that
$$R_\pi(T) \geq 1 - k_1k_2\mu \sqrt{\frac{m}{16(n_1 - k_1)(n_2 - k_2)}}$$
proving the lower bound on the minimax risk.

\subsection{Proof of Theorem \ref{thm:detectionub}}

Define $t = \frac{1}{\sqrt{m}} \sum_{i = 1}^m y_i$. It is easy to see
that under $H_0$, $t \sim \mathcal{N}(0,\sigma^2)$ while under $H_1$,
$t \sim \mathcal{N}( \sqrt{\frac{m}{n_1n_2}} k_1k_2 \mu, \sigma^2)$.
The theorem now follows from an application of standard Gaussian tail
bounds in Eq.~\eqref{eq:tail-bound-normal}. 

\subsection{Proof of Theorem \ref{thm:activelb}}
The proof will proceed via two separate constructions. At a high level
these constructions are intended to capture the difficulty of
exactly and approximately localizing the activation block.

{\bf Construction 1 - approximate localization:} Let us define three
distributions: $\mathbb{P}_0$ corresponding to no bicluster,
$\mathbb{P}_1$ which is a uniform mixture over the distributions
induced by having the top-left corner of the bicluster in the left
half of the matrix and $\mathbb{P}_2$ which is a uniform mixture over
the distributions induced by having the top-left corner of the
bicluster in the right half of the matrix.

We first upper bound the total variation between $\mathbb{P}_1$ and
$\mathbb{P}_2$.  This results directly in a lower bound for the
problem of distinguishing whether the top-left corner of the bicluster
is in the left or right half of the matrix, which in turn is a lower
bound for the localization of the bicluster.

Now notice that,
\begin{eqnarray*}
||\mathbb{P}_1 - \mathbb{P}_2||_{TV}^2 & \leq & 2||\mathbb{P}_0 - \mathbb{P}_1||_{TV}^2 + 2 ||\mathbb{P}_0 - \mathbb{P}_2||_{TV}^2 \\
& \leq & KL(\mathbb{P}_0,\mathbb{P}_1) + KL(\mathbb{P}_0, \mathbb{P}_2) \\
\end{eqnarray*}
Notice that $KL(\mathbb{P}_0,\mathbb{P}_1)$ is exactly the quantity we
have to upper bound to produce a lower bound on the signal strength
for detecting whether a block of activation is in the left half of the
matrix or not.  At least from a lower bound perspective this reduces
the problem of localization to that of detection. We can now apply a
slight modification of the proof of Theorem \ref{thm:detectionlb} to
obtain that
\begin{eqnarray*}
KL(\mathbb{P}_0,\mathbb{P}_1) = KL(\mathbb{P}_0,\mathbb{P}_2 ) \leq
\frac{m \mu^2 k_1^2 k_2^2}{(n_1 - k_1) (n_2/2 - k_2)} \\
\end{eqnarray*}

Noting that the minimax risk $R$ for distinguishing $\mathbb{P}_1$ from $\mathbb{P}_2$
\begin{eqnarray*}
R = 1 - \frac{1}{2} ||\mathbb{P}_1 - \mathbb{P}_2||_{TV} \geq 1 -
\sqrt{\frac{m \mu^2 k_1^2 k_2^2}{2(n_1 - k_1) (n_2/2 - k_2)}}
\end{eqnarray*}

{\bf Construction 2 - exact localization: } Without loss of generality
we assume $k_1 \leq k_2$.  Consider, two distributions $\mathbb{P}_1$
and $\mathbb{P}_2$, where $\mathbb{P}_1$ is induced by matrix $A_1$
when the activation block $B = B_1 = [1,\ldots,k_1][1,\ldots,k_2]$ and
$\mathbb{P}_2$ is induced by matrix $A_2$ when the activation block $B
= B_2 = [1,\ldots,k_1][2,\ldots,k_2+1]$.

Now, following the same argument as in the proof of Theorem
\ref{thm:detectionlb}, we have
\begin{eqnarray*}
KL(\mathbb{P}_1,\mathbb{P}_2) & = & \mathbb{E}_{\mathbb{P}_1} \sum_{i = 1}^m \left( - \frac{1}{2\sigma^2} \left[ (y_i - \mathrm{tr}(A_1X_i))^2 - (y_i - \mathrm{tr}(A_2X_i))^2  \right]\right) \\
& = & \frac{1}{2\sigma^2} \mathbb{E}_{\mathbb{P}_1} \sum_{i = 1}^m \left[ \mathrm{tr}(A_2X_i)^2 - \mathrm{tr}(A_1X_i)^2 + 2 y_i \mathrm{tr}(A_1X_i) - 2 y_i \mathrm{tr}(A_2X_i) \right] \\
& = & \frac{1}{2\sigma^2} \mathbb{E}_{\mathbb{P}_1} \sum_{i = 1}^m\left(  \underbrace{\mathrm{tr}(A_2X_i) - \mathrm{tr}(A_1X_i) }_{t_i}\right)^2  =  \frac{1}{2\sigma^2} \mathbb{E}_{\mathbb{P}_1} \sum_{i = 1}^m t_i^2 \\
\end{eqnarray*}

Now, with some abuse of notation,
\begin{eqnarray*}
t_i & = & \mu \left( \sum_{j \in B_1\backslash B_2} X_{ij} - \sum_{j
    \in B_2 \backslash B_1} X_{ij}  \right) \\
& \leq & \mu \left( \sum_{j \in B_1 \Delta B_2} |X_{ij}| \right) \\
\end{eqnarray*}
By using Cauchy-Schwarz we get
\begin{eqnarray*}
t_i^2 \leq 2\mu^2 k_1 \sum_{j \in B_1 \Delta B_2} X^2_{ij} \leq 2\mu^2 k_1
\end{eqnarray*}
since $||X_i||_F^2 = 1$.

This gives us that,
\begin{eqnarray*}
KL(\mathbb{P}_1,\mathbb{P}_2) \leq \frac{m k_1 \mu^2}{\sigma^2}
\end{eqnarray*}
Together with a similar construction for the case when $k_2 \leq k_1$ we get
$$KL(\mathbb{P}_1,\mathbb{P}_2) \leq \frac{m \min(k_1,k_2) \mu^2}{\sigma^2}$$

Once again noting (by Pinsker's theorem),
\begin{eqnarray*}
R \geq 1 - \sqrt{KL(\mathbb{P}_1,\mathbb{P}_2)/8} \geq 1 -  \sqrt{\frac{m \min(k_1,k_2) \mu^2}{8\sigma^2}}
\end{eqnarray*}

Combining the approximate and exact localization bounds we get,
$$R \geq \max\left( 1 -  \sqrt{\frac{m \min(k_1,k_2) \mu^2}{8\sigma^2}}, 1 - \sqrt{\frac{m \mu^2 k_1^2 k_2^2}{2(n_1 - k_1) (n_2/2 - k_2)}} \right) $$

Thus, we get for any $0 < \alpha < 1$, $R \geq \alpha$ if $$\min\left( \sqrt{\frac{m \min(k_1,k_2) \mu^2}{8\sigma^2}},\sqrt{\frac{m \mu^2 k_1^2 k_2^2}{2(n_1 - k_1) (n_2/2 - k_2)}} \right) \leq 1 - \alpha$$

\subsection{Proof of Theorem \ref{thm:passivelb} }

Without loss of generality we assume $k_1 \leq k_2$.  Consider, two
distributions $\mathbb{P}_1$ and $\mathbb{P}_2$, where $\mathbb{P}_1$
is induced by matrix $A_1$ when the activation block $B = B_1 =
[1,\ldots,k_1]\times[1,\ldots,k_2]$ and $\mathbb{P}_2$ is induced by matrix
$A_2$ when the activation block $B = B_2 =
[1,\ldots,k_1]\times[2,\ldots,k_2+1]$.

Following the proof of Theorem \ref{thm:activelb}.
\begin{equation}
\label{eq:kl:1}
  \begin{aligned}
\mathrm{KL}(\mathbb{P}_1, \mathbb{P}_2) 
& =  \mathbb{E}_{\mathbb{P}_1} \log \frac{\mathbb{P}_1}{\mathbb{P}_2} \\
& = \frac{1}{2\sigma^2} \mathbb{E}_{\mathbb{P}_1} \sum_{i=1}^m \left(\mathrm{tr}(A_2X_i) - \mathrm{tr}(A_1X_i) \right)^2 \\
& = \frac{\mu^2}{\sigma^2} \frac{mk_1}{n_1n_2}
\end{aligned}
\end{equation}
using the fact that $X_i$ is a random Gaussian matrix with independent
entries of variance $\frac{1}{n_1n_2}$.

Now, note that the minimax risk
$$R \geq 1 - \sqrt{\mathrm{KL}(\mathbb{P}_1, \mathbb{P}_2)/8}.$$

For the second part of the theorem, we consider $\mathbb{P}_2, \ldots,
\mathbb{P}_{t+1}$, where $t = (n_1 - k_1)(n_2 - k_2)$, each of which
is induced by a $B$ which does not overlap with $B_1$.

The same calculation now gives
\begin{equation}
\label{eq:kl:2}
\mathrm{KL}(\mathbb{P}_1, \mathbb{P}_j) \leq \frac{\mu^2}{\sigma^2} \frac{mk_1k_2}{n_1n_2}
\end{equation}
Now, applying the multiple hypothesis version of Fano's inequality
\citep[see Theorem 2.5 in][]{tsybakov09introduction} we conclude the proof.

\subsection{Proof of Theorem~\ref{thm:passiveub}}

Let $z_{i,B} = \sum_{(a,b) \in B} X_{i,ab}$ and $\zb_B = (z_{1,B},
\ldots, z_{m,B})'$. With this, we can write the loss function defined
in Eq.~\eqref{eq:score_function} as
\begin{equation}
  \label{eq:score_function:1}
  f(B) := \min_{\hat\mu_B} \norm{\hat\mu_B \zb_B - \yb}_2^2.
\end{equation}
Let $\Delta(B) = f(B) - f(B^*)$ and observe that an error is made if
$\Delta(B) < 0$ for $B \neq B^*$. Therefore,
\[
\PP[\text{error}] = \PP[\cup_{B \in \Bcal\bks B^*} \{\Delta(B) < 0\}].
\]
Under the conditions of the theorem, we will show that $\Delta(B) > 0$
for all $B \in \Bcal\bks B^*$ with large probability.

The following lemma shows that for any fixed $B$, the event $\{
\Delta(B) < 0 \}$ occurs with exponentially small probability. 
\begin{lemma}
  \label{lem:passive:bound_delta}
  Fix any $B \in \Bcal\bks B^*$. Then
  \begin{equation}
    \label{eq:passive:bound_delta}
    \PP[\Delta(B) < 0] \leq 
\exp\left(
-c_1
\frac{\mu^2m|B^*\bks B|}{\sigma^2n_1n_2}
\right)  +
 c_2 \exp(-c_3m).
  \end{equation}
\end{lemma}

From the second term in Eq.~\eqref{eq:passive:bound_delta}, we obtain
a lower bound on the sample size $m$. Using the union bound, it is
sufficient that $m$ satisfies 
\[
c_1(n_1-k_1)(n_2-k_2)\exp(-c_2m) \leq \delta/2,
\]
which gives us the lower bound as $m \geq C \log \max(n_1 - k_1, n_2 -
k_2)$.

Define $N(l) = |\{B \in \Bcal \ :\ |B \Delta B^*| = l\}|$ to be the
number of elements in $\Bcal$ whose symmetric difference with
$B^*$ is equal to $l$. Note that $N(l) = \Ocal(1)$ for any $l$. Using
the union bound
\begin{equation}
\label{eq:proof:passive_loc:union_bound}
\begin{aligned}
  & \PP[\cup_{B \in \Bcal} \{\Delta(B) < 0\}] \\
  & \leq
    \sum_{B \in \Bcal, |\bdb| = 2k_1k_2} 
      \exp\left(
        -c_1\frac{\mu^2k_1k_2m}{\sigma^2{n_1n_2}}
      \right) +
    \sum_{l < 2k_1k_2}
    N(l)
      \exp\left(
        -c_1\frac{\mu^2lm}{\sigma^2{n_1n_2}}
      \right)
     \\
  & \leq c_2(n_1-k_1)(n_2-k_2)
      \exp\left(
        -c_1\frac{\mu^2k_1k_2m}{\sigma^2{n_1n_2}}
      \right) +
    c_3k_1k_2
      \exp\left(
        -c_1\frac{\mu^2\min(k_1,k_2)m}{\sigma^2{n_1n_2}}
      \right).
\end{aligned}
\end{equation}
Choosing 
\[
\mu = c_1 \sigma \sqrt{\frac{n_1n_2}{m} \log(2/\delta)
  \max\rbr{\frac{\log\max(k_1,k_2)}{\min(k_1,k_2)}, 
       \frac{\log\max(n_1-k_1,n_2-k_2)}{k_1k_2}}
}
\]
each term in Eq.~\eqref{eq:proof:passive_loc:union_bound} will be
smaller than $\delta/2$, with an appropriately chosen constant $c_1$.

We finish the proof of the theorem, by proving
Lemma~\ref{lem:passive:bound_delta}.

\begin{proof}[Proof of Lemma~\ref{lem:passive:bound_delta}]
For any $B \in \Bcal$, let
\[
\begin{aligned}
\hat\mu_B 
&= \argmin_{\hat\mu_B} \norm{\hat\mu_B \zb_B - \yb}_2^2\\
&= \norm{\zb_B}_2^{-2}\zb_B'\yb.
\end{aligned}
\]
Note that $ \hat{\mu}_{B^*} = \mu + \norm{\zb_{B^*}}_2^{-2}\zb_{B^*}'\epsilonb$.

Let 
\begin{equation*}
  \begin{aligned}
    \Hb_B& = \norm{\zb_B}_2^{-2}\zb_B\zb_B'  \\
    \Hb_B^{\perp} &= \Ib - \norm{\zb_B}_2^{-2}\zb_B\zb_B'
  \end{aligned}
\end{equation*}
be the projection matrices and write
\[
\begin{aligned}
f(B^*) &= \norm{\Hb_{B^*}^\perp\epsilonb}_2^2 \\
f(B) &= \norm{\Hb_B^\perp(\zb_{B^*} \mu + \epsilonb)}_2^2 
= \norm{\Hb_{B}^\perp\epsilonb}_2^2 +
\mu^2\norm{\Hb_B^\perp\zb_{B^*}}_2^2
+2\epsilonb'\Hb_B^\perp\zb_{B^*} \mu.
\end{aligned}
\]
Now,
\[
\Deltab(B)  = \underbrace{\norm{\Hb_{B}^\perp\epsilon}_2^2 -
\norm{\Hb_{B^*}^\perp\epsilon}_2^2}_{T_1} + 
\underbrace{\mu^2\norm{\Hb_B^\perp\zb_{B^*}}_2^2
+2\epsilonb'\Hb_B^\perp\zb_{B^*} \mu}_{T_2}.
\]
Conditional on $\Xb$, $\norm{\Hb_{B}^\perp\epsilon}_2^2 \mid \Xb \sim
\sigma^2\chi^2_{m-1}$ and $\norm{\Hb_{B^*}^\perp\epsilon}_2^2 \mid \Xb
\sim \sigma^2\chi^2_{m-1}$ \citep[see Theorem 3.4.4
in][]{mardia1980multivariate} . Since the conditional distributions do
not depend on $\Xb$, they are the same as the marginal
distributions. Therefore, $T_1 \sim \sigma^2(V_1-V_2)$ where $V_1, V_2
\sim \chi^2_{m-1}$.
\begin{align}
\label{eq:proof:bound_t1}
  \PP\left[|T_1| \geq \frac{\sigma^2(m-1)\eta}{2}\right] 
  \leq
  2\PP\left[|\chi_{m-1}^2 - m + 1|\geq
  \frac{(m-1)\eta}{4}\right] 
  \leq 2\exp\left(-\frac{3(m-1)\eta^2}{256}\right)
\end{align}
using Eq.~\eqref{eq:chi-central-upper-johnstone}, as long as $\eta
\in [0, 2)$.

To analyze the term $T_2$, we condition on $\Xb$, so that 
\begin{equation*}
  T_2|\Xb \sim \Ncal(\tilde\mu, 4\sigma^2\tilde\mu)
\end{equation*}
where $\tilde\mu = \mu^2\norm{\Hb_B^\perp\zb_{B^*}}_2^2$. This
gives
\[
\PP[T_2  \leq \tilde\mu/2| \Xb] = 
\PP[\Ncal(0,1) \geq \sqrt{\tilde\mu}/(4\sigma) | \Xb].
\]
Next, we show how to control $\norm{\Hb_B^\perp\zb_{B^*}}_2^2$.
Writing $\zb_{B^*} = \zb_B - \zb_{B \bks B^*} + \zb_{B^* \bks B}$,
simple algebra gives
\begin{equation*}
\begin{aligned}
&\norm{\Hb_B^\perp\zb_{B^*}}_2^2 \\
&= 
\norm{\Hb_B^\perp\zb_{B^* \bks B}}_2^2 +
\norm{\Hb_B^\perp\zb_{B \bks B^*}}_2^2 -
2\zb_{B^* \bks B}'\Hb_B^\perp\zb_{B \bks B^*} \\
& = \norm{\Hb_B^\perp\zb_{B^* \bks B}}_2^2 +
\norm{\zb_{B\bks B^*}-\zb_{B^* \bks B}}_2^2 -
\norm{\zb_{B^* \bks B}}_2^2 -
\frac{
((\zb_{B\bks B^*}-\zb_{B^* \bks B})'\zb_B)^2 - 
(\zb_{B^* \bks B}'\zb_B)^2
}{\norm{\zb_{B}}_2^2}\\
& \geq \norm{\Hb_B^\perp\zb_{B^* \bks B}}_2^2 +
\norm{\zb_{B\bks B^*}-\zb_{B^* \bks B}}_2^2 -
\norm{\zb_{B^* \bks B}}_2^2 -
\frac{
((\zb_{B\bks B^*}-\zb_{B^* \bks B})'\zb_B)^2
}{\norm{\zb_{B}}_2^2}.
\end{aligned}
\end{equation*}
Define the event 
\begin{equation*}
\begin{aligned}
\Ecal(\eta) = & 
\left\{ 
\norm{\Hb_B^\perp\zb_{B^* \bks B}}_2^2
\geq \frac{(1-\eta)(m-1)|B^*\bks B|}{n_1n_2}
\right\}
\bigcap
\left\{ 
\norm{\zb_{B\bks B^*}-\zb_{B^* \bks B}}_2^2
\geq \frac{(1-\eta)2m|B^*\bks B|}{n_1n_2}
\right\} \\
& \bigcap
\left\{ 
\norm{\zb_{B^* \bks B}}_2^2
\leq \frac{(1+\eta)m|B^*\bks B|}{n_1n_2}
\right\} 
\bigcap
\left\{ 
\norm{\zb_{B}}_2^2
\geq \frac{(1-\eta)m|B|}{n_1n_2}
\right\} \\
&\bigcap
\left\{ 
|(\zb_{B\bks B^*}-\zb_{B^* \bks B})'\zb_B|
\leq \frac{(1+\eta)m|B^*\bks B|}{n_1n_2}
\right\},
\end{aligned}
\end{equation*}
such that, using the concentration results in Appendix B, 
\[
\PP[\Ecal(\eta)^C] \leq c_1 \exp(-c_2m\eta^2).
\]
On the event $\Ecal(\eta)$ we have that 
\begin{equation*}
\begin{aligned}
\norm{\Hb_B^\perp\zb_{B^*}}_2^2 &
\geq
\frac{m|B^*\bks B|}{n_1n_2}
\left[3(1-\eta)-(1+\eta)-
\frac{(1+\eta)^2}{1-\eta}\frac{|B^*\bks B|}{|B|}\right]
-\frac{(1-\eta)|B^*\bks B|}{n_1n_2}\\
& \geq
c_1
\frac{m|B^*\bks B|}{n_1n_2}.
\end{aligned}
\end{equation*}
Therefore,
\begin{equation}
\label{eq:proof:bound_t2}
  \begin{aligned}
\PP[T_2 \leq \tilde\mu/2|\Xb] & \leq
\PP\left[\Ncal(0,1)\geq c_1\frac{\mu}{\sigma}
\sqrt{
\frac{m|B^*\bks B|}{n_1n_2}}\right]
+ \PP[\Ecal^C]  \\
& \leq
\exp\left(
-c_1
\frac{\mu^2m|B^*\bks B|}{\sigma^2n_1n_2}
\right)  +
 c_2 \exp(-c_3m\eta^2).
  \end{aligned}
\end{equation}
Combining Eq.~\eqref{eq:proof:bound_t1} and
Eq.~\eqref{eq:proof:bound_t2} completes the proof.

\end{proof}

\subsection{Proof of Theorem \ref{thm:activeub}}
As with the lower bound the localization algorithm and analysis is
naturally divided into two phases. An approximate localization phase
and an exact localization one. We will analyze each of these in
turn. To ease presentation we will assume $n_1$ is a dyadic multiple
of $2k_1$ and $n_2$ a dyadic multiple of $2k_2$.  Straightforward
modifications are possible when this is not the case.

{\bf Approximate localization: } The approximate localization phase
proceeds by a modification of the compressive binary search (CBS)
procedure of \citet{nowakcbs} (see also \citet{davenportcbs}) on the
matrix $A$.

We will run this modified CBS procedure four times on sets of
blocks of the matrix $A$. The four sets are
\begin{eqnarray*}
\mathcal{D}_1 & \equiv & \left\{ B_{1,1} := [1,\ldots,2k_1] \times [1,\ldots,2k_2]
  , B_{1,2} := [2k_1+1,\ldots,4k_1] \right. 
 \times [1,\ldots,2k_2]  \\
 & &
\left. \ldots , B_{1,n_1n_2/4k_1k_2}
:= [n_1 - 2k_1, \ldots,n_1] \times [n_2-2k_2,\ldots,n_2] \right\} \\
\mathcal{D}_2 & \equiv & \left\{ B_{2,1} :=
  [k_1,\ldots,3k_1] \times [k_2,\ldots,3k_2] , B_{2,2} :=
  [3k_1+1,\ldots,5k_1] \times [k_2,\ldots,3k_2] \right. \\
 & & \left.  \ldots , B_{2,n_1n_2/4k_1k_2} :=
  [n_1-k_1,...,n_1,1,\ldots,k_1] \times [n_2
  -k_2,...,n_2,1,\ldots,k_2]\right\} \\
\mathcal{D}_3 & \equiv & \left\{ B_{3,1} :=
  [k_1,\ldots,3k_1] \times [1,\ldots,2k_2] , B_{3,2} :=
  [3k_1+1,\ldots,5k_1] \times [1,\ldots,2k_2] \right. \\
 & & \left.  \ldots , B_{3,n_1n_2/4k_1k_2} :=
  [n_1-k_1,...,n_1,1,\ldots,k_1] \times [n_2-2k_2,\ldots,n_2]\right\} 
  \end{eqnarray*}
  and
  \begin{eqnarray*}
  \mathcal{D}_4 & \equiv & \left\{ B_{4,1} :=
  [1,\ldots,2k_1] \times [k_2,\ldots,3k_2] , B_{4,2} :=
  [2k_1+1,\ldots,4k_1] \times [k_2,\ldots,3k_2] \right. \\
 & & \left.  \ldots , B_{4,n_1n_2/4k_1k_2} :=
  [n_1 - 2k_1, \ldots,n_1] \times [n_2
  -k_2,...,n_2,1,\ldots,k_2]\right\}.
\end{eqnarray*}
Notice that the entire block of activation is always \emph{fully}
contained in one of these blocks.  The output of the CBS procedure
when run on these four collections is four blocks - one from each
collection. We define an
approximate localization \emph{error} to be the event in which none
of the blocks returned fully contains the block of activation.

Without loss of generality let us assume that the activation block is
fully contained in some block from the first collection. Once we have
fixed the collection of blocks the CBS procedure is invariant to
reordering of the blocks, so without loss of generality we can
consider the case when the activation block is contained in $B_{11}$.

The analysis proceeds exactly as in \cite{nowakcbs}. We only outline
the differences arising from having a block of activation as opposed to a
single activation in a vector, and refer the reader to \cite{nowakcbs} 
for the details. 

The binary search
procedure on the first collection of blocks proceeds for $$s_0 \equiv
\log \left( \frac{n_1n_2}{4k_1k_2} \right)$$ rounds.  Now, we can
bound the probability of error of the procedure by a union bound as
$$\mathbb{P}_e \leq \sum_{s = 1}^{s_0} P[w^{s} < 0]$$
where $$w^s \sim \mathcal{N}\left(\frac{m_s 2^{(s-1)/2} k_1 k_2 \mu}{\sqrt{n_1n_2}}, m_s \sigma^2\right)$$

Recall, the allocation scheme: for $m \geq 2 s_0$,
$m_s \equiv \lfloor (m - s_0) s 2^{-s - 1}\rfloor + 1$
and observe that
$\sum_{s = 1}^{s_0} m_s \leq m$ 

Now, using the Gaussian tail bound 
$$P[N(0,1) > t] \leq \frac{1}{2} \exp (-t^2/2)$$ 
we see that
$$\mathbb{P}_e \leq \frac{1}{2} \sum_{s = 1}^{s_0} \exp \left( -
\frac{m_s 2^{s} k_1^2 k_2^2 \mu^2}{4n_1n_2 \sigma^2}
\right)$$

Now, observe that $m_s \geq (m - s_0) s 2^{-s - 1}$ and $m \geq 2s_0$, so 
$m_s \geq ms2^{-s-2}$.

It is now straightforward to verify that if
$$\mu \geq \sqrt{ \frac{16\sigma^2n_1n_2}{mk_1^2k_2^2} \log \left( \frac{1}{2\delta} + 1 \right)}$$
we have $\mathbb{P}_e \leq \delta$. We apply this procedure 4 times (once on each collection).

Let us revisit what we have shown so far: if $\mu$ is large enough then one of the four
runs of the CBS procedure will return a block of size $(2k_1 \times 2k_2)$ which
fully contains the block of activation, with probability at least $1 - 4\delta$.

{\bf Exact localization: } 
We collect all the rows and columns returned by the 4 runs of the CBS procedure.
In the $1- 4\delta$ probability event described above,
we have a block of at most $(8k_1 \times 8k_2)$ which contains the full block
of activation (for simplicity we disregard the fact that we know that the block
is actually in one of two $(4k_1 \times 4k_2)$ blocks, i.e. we assume the worst case
that none of the returned blocks overlap in their rows or columns and we explore the off-diagonal
blocks).

Let us first identify the active columns. First, notice that exactly one of the
following columns: $\{1, k_2 + 1, 2k_2 + 1,\ldots, 7k_2 + 1\}$ must
be active.

Let us devote $8m$ measurements to identifying the active column
amongst these. The procedure is straightforward: measure each column
$m$ times, and pick the one that has the largest total signal.

It is easy to show that the active column results in a draw from
$\mathcal{N}(\sqrt{\frac{k_1}{8}} \mu m, m\sigma^2)$ and the non-active columns
result in draws from $\mathcal{N}(0, m\sigma^2)$.

Using the same Gaussian tail bound as before it is easy to show that if
$$\mu \geq \sqrt{ \frac{64 \sigma^2}{k_1m} \log (4/\delta)}$$
we successfully find the active column with probability at least $1 - \delta$.

So far, we have identified an active column and localized the 
columns of the activation block to one of $2k_2$ columns. We will use
$m$ more measurements to find the remaining active columns. Rather,
than test each of the $2k_2$ columns we will do a binary search. This will
require us to test at most $t \equiv 2 \lceil \log k_2 \rceil \leq 3\log k_2$ 
columns, and we will devote $m/ (3\log k_2)$ measurements to each
column. We will need to threshold these
measurements at 
$$ \sqrt{ \log \left( \frac{3 \log k_2}{\delta} \right) \frac{2m\sigma^2}{3 \log k_2}} $$
and declare a row as active if its average is larger than this.

It is easy to show that this binary search procedure successfully finds
all active columns with
probability at least $1 - \delta$ if
$$\mu \geq \sqrt { \frac{32 \sigma^2 \log k_2}{mk_1} \log \left( \frac{3 \log k_2}{  \delta} \right)} $$

We repeat this procedure to identify the active rows.

{\bf Putting everything together: }
Total number of measurements used:
\begin{enumerate}
\item Four rounds of CBS: $4m$
\item Identifying first active column and first active row: $16m$
\item Identifying remaining active rows and columns: $2m$
\end{enumerate}
This is a total of $22m$ measurements.
Each of these steps fails with a probability at most $\delta$, for a total of $8\delta$.

Now, re-adjusting constants we obtain, if
$$ \mu \geq \max \left( \sqrt{ \frac{352\sigma^2n_1n_2}{mk_1^2k_2^2} \log \left( \frac{4}{\delta} + 1 \right)},
\sqrt { \frac{1408\sigma^2 \log \max(k_1,k_2)}{m \min(k_1,k_2)} \log \left( \frac{24 \log \max(k_1, k_2)}{  \delta} \right)}\right)$$
then we successfully localize the matrix with probability at least $1-\delta$.

Stated more succinctly we require
$$\mu \geq  \tilde{O} \left( \max \left( \sqrt{\frac{\sigma^2n_1n_2}{m k_1^2 k_2^2}}, \sqrt{\frac{\sigma^2}{\min(k_1,k_2)m}}\right)\right).$$
This matches the lower bound up to $\log k$ factors.

\subsection{Proof of Eq.~\eqref{eq:upper_bound_bicluster_passive} and
  Eq.~\eqref{eq:lower_bound_bicluster_passive}}

\label{sec:app:proof_sketch}

Proof of Eq.~\eqref{eq:upper_bound_bicluster_passive} follows the same
line as the proof of Theorem~\ref{thm:passiveub}. We have
\[
\begin{aligned}
\PP[\text{error}] 
& = \PP[\cup_{B \in \Bcal\bks B^*} \{\Delta(B) < 0\}] \\
     & \leq \sum_{i = 0}^{k_1} 
         {k_1 \choose i} {n_1 - k_1 \choose k_1- i} 
          \sum_{j=0}^{k_2}
         {k_2 \choose j} {n_2 - k_2 \choose k_2- j} 
\exp\left(
-c_1
\frac{(\mu^*)^2m(k_1k_2 - ij)}{\sigma^2n_1n_2}
\right)  
\\
     & \ \ + \sum_{i = 0}^{k_1} 
         {k_1 \choose i} {n_1 - k_1 \choose k_1- i} 
          \sum_{j=0}^{k_2}
         {k_2 \choose j} {n_2 - k_2 \choose k_2- j}
 c_2 \exp(-c_3m).
\end{aligned}
\]
The argument given in the proof of Theorem 2 in \cite{kolarBRS11}
gives us Eq.~\eqref{eq:upper_bound_bicluster_passive} if $m \geq C
\log \max \left({n_1 \choose k_1}, {n_2 \choose k_2}\right)$. Proof of
Eq.~\eqref{eq:lower_bound_bicluster_passive} follows the proof of
Theorem 1 in  \cite{kolarBRS11} with the appropriate KL divergences
derived in Eq.~\eqref{eq:kl:1} and Eq.~\eqref{eq:kl:2}.

\section{Collection of concentration results}

In this section, we collect useful results on tail bounds of various
random quantities used throughout the paper. We start by stating a
lower and upper bound on the survival function of the standard normal
random variable. Let $Z \sim \Ncal(0,1)$ be a standard normal random
variable. Then for $t > 0$
\begin{equation}
  \label{eq:tail-bound-normal}
  \frac{1}{\sqrt{2\pi}} \frac{t}{t^2+1}\exp(-t^2/2)
  \leq \PP(Z > t) 
  \leq \frac{1}{\sqrt{2\pi}}\frac{1}{t} \exp(-t^2/2).
\end{equation}

\subsection{Tail bounds for Chi-squared variables}

Throughout the paper we will often use one of the following tail
bounds for central $\chi^2$ random variables. These are well known and
proofs can be found in the original papers.

\begin{lemma}[\citet{Laurent00adaptive}] Let $X \sim \chi^2_d$. For all
  $x \geq 0$,
\begin{align}
  \label{eq:chi-central-upper}
  \PP[X - d \geq 2 \sqrt{dx} + 2x] & \leq \exp(-x)\\
  \label{eq:chi-central-lower}
  \PP[X - d \leq -2 \sqrt{dx}] & \leq \exp(-x).
\end{align}
\end{lemma}

\begin{lemma}[\citet{johnstone2009consistency}]
  Let $X \sim \chi^2_d$, then
  \begin{equation}
    \label{eq:chi-central-upper-johnstone}
    \mathbb{P}[|d^{-1}X -1| \geq x] \leq \exp(-\frac{3}{16}dx^2), \quad x
    \in [0, \frac{1}{2}).  
  \end{equation}
\end{lemma}

The following result provide a tail bound for non-central $\chi^2$
random variable with non-centrality parameter $\nu$.

\begin{lemma}[\citet{birge2001alternative}] Let $X \sim
  \chi^2_d(\nu)$, then for all $x > 0$
\begin{align}
  \label{eq:chi-noncentral-upper}
  \PP[X \geq (d+\nu) + 2 \sqrt{(d+2\nu)x} + 2x] & \leq \exp(-x)\\
  \label{eq:chi-noncentral-lower}
  \PP[X \leq (d+\nu)-2 \sqrt{(d+2\nu)x}] & \leq \exp(-x).
\end{align}
\end{lemma}

Using the above results, we have a tail bound for sum of
product-normal random variables.

\begin{lemma} 
  \label{lem:deviation-covariance}
  Let $Z = (Z_a, Z_b) \sim
  \Ncal_2(0,0,\sigma_{aa},\sigma_{bb},\sigma_{ab})$ be a bivariate
  Normal random variable and let $(z_{ia},z_{ib}) \iidsim Z$, $i=1,\ldots,n$.
  Then for all $t \in [0,\nu_{ab}/2)$ 
\begin{eqnarray} 
  \mathbb{P}\left[ \left|n^{-1}\sum_i z_{ia}z_{ib} -
      \sigma_{ab} \right| \geq t \right] \leq
  4 \exp\left(-\frac{3nt^2}{16\nu_{ab}^2}\right),
\end{eqnarray}
where $\nu_{ab} = \max \{(1-\rho_{ab})\sqrt{\sigma_{aa}\sigma_{bb}},
  (1+\rho_{ab})\sqrt{\sigma_{aa}\sigma_{bb}} \}$.
\end{lemma}
\begin{proof}
  Let $z_{ia}' = z_{ia} / \sqrt{\sigma_{aa}}$. Then using
  \eqref{eq:chi-central-upper-johnstone}
  \begin{equation*}
  \begin{aligned}
    \PP&[|\frac{1}{n} \sum_{i=1}^n z_{ia}z_{ib} - \sigma_{ab}| \geq t]
    \\ &=
    \PP[|\frac{1}{n} \sum_{i=1}^n z'_{ia} z'_{ib} - \rho_{ab}| 
    \geq \frac{t}{\sqrt{\sigma_{aa}\sigma_{bb}}}]
    \\ & = 
    \PP[|\sum_{i=1}^n ((z'_{ia} + z'_{ib})^2  - 2(1+\rho_{ab}))
    - ((z'_{ia} - z'_{ib})^2  - 2(1-\rho_{ab}))| 
    \geq \frac{4nt}{\sqrt{\sigma_{aa}\sigma_{bb}}}] \\
    & \leq
    \PP[|\sum_{i=1}^n ((z'_{ia} + z'_{ib})^2  - 2(1+\rho_{ab}))| 
    \geq \frac{2nt}{\sqrt{\sigma_{aa}\sigma_{bb}}}] 
    \\ & \quad+
    \PP[|\sum_{i=1}^n((z'_{ia} - z'_{ib})^2  - 2(1-\rho_{ab}))| 
    \geq \frac{2nt}{\sqrt{\sigma_{aa}\sigma_{bb}}}] 
    \\ & \leq
    2\PP[|\chi^2_n - n| \geq \frac{nt}{\nu_{ab}}]    
     \leq 4 \exp(-\frac{3nt^2}{16\nu_{ab}^2}),
  \end{aligned}
  \end{equation*}
  where $\nu_{ab} = \max \{(1-\rho_{ab})\sqrt{\Sigma_{aa}\Sigma_{bb}},
  (1+\rho_{ab})\sqrt{\Sigma_{aa}\Sigma_{bb}} \}$ and $t \in [0,
  \nu_{a}/2)$.
\end{proof}

\begin{corollary}
  \label{lem:deviation-product-gaussian}
  Let $Z_1$ and $Z_2$ be two independent standard Normal random
  variables and let $X_i \iidsim Z_1Z_2$, $i=1\ldots n$. Then for $t
  \in [0, 1/2)$
  \begin{equation}
    \label{eq:deviation-product-gaussian}
    \PP[|n^{-1}\sum_{i \in [n]} X_i| > t] \leq
    4\exp\rbr{-\frac{3nt^2}{16}}.
  \end{equation}
\end{corollary}


\end{document}